\documentclass{article}
\usepackage{ijcai16}

\usepackage{times}
\usepackage[lined, boxed, linesnumbered]{algorithm2e}
\usepackage{amsfonts}
\usepackage{amsthm}
\usepackage{float}
\usepackage{mathtools}
\usepackage{graphicx}
\usepackage{blindtext}

\newtheorem{theorem}{Theorem}[section]
\newtheorem{corollary}{Corollary}[theorem]
\newtheorem{lemma}[theorem]{Lemma}

\title{Time Resource Networks}
\author{Szymon Sidor  \and Peng Yu  \and Cheng Fang \and Brian Williams\\
Computer Science \& Artificial Intelligence Lab, MIT, Cambridge, MA\\
\texttt{\{sidor,yupeng,cfang,williams\}@mit.edu}}
\begin{document}

\maketitle

\begin{abstract}
  The problem of scheduling under resource constraints is widely applicable. One prominent example is power management, in which we have a limited continuous supply of power but must schedule a number of power-consuming tasks. Such problems feature tightly coupled continuous resource constraints and continuous temporal constraints.

  We address such problems by introducing the Time Resource Network (TRN), an encoding for resource-constrained scheduling problems. The definition allows temporal specifications using a general family of representations derived from the Simple Temporal network, including the Simple Temporal Network with Uncertainty, and the probabilistic Simple Temporal Network (Fang et al. (2014)).

  We propose two algorithms for determining the consistency of a TRN: one based on Mixed Integer Programing and the other one based on Constraint Programming, which we evaluate on scheduling problems with Simple Temporal Constraints and Probabilistic Temporal Constraints.
\end{abstract}
%%%%%%%%%%%%%%%%%%%%%%%%%%%%%% INTRODUCTION %%%%%%%%%%%%%%%%%%%%%%%%%%%%%%%%%%%
\section{Introduction}
Temporal Networks scheduling algorithms support diverse formulations useful in modeling practical problems. Examples include dynamical execution strategies based on partial knowledge of uncertain durations, and strategies to upper-bound the probability of failing to satisfy temporal constraints given distributions over uncertain durations. However, it is not obvious how to apply them in scenarios with resource usage constraints. While some prior work exists in operations research literature, known as project scheduling or job-shop scheduling, much of the focus is on discrete resources. We attempt to narrow the gap between the two independent bodies of work.

As a motivating example, consider the following Smart House scenario. A $150W$ generator is available, and we know that the resident returns home at some time defined by a Normal distribution $N(5pm, 5\ \text{minutes})$. Moreover we know that sun sets at time defined by $N(7pm, 1\ \text{minute})$. We would like to meet the following constraints with the overall probability at least $98\%$:
\small{
  \begin{itemize}
  \setlength\itemsep{0.00em}

  \item Wash clothes (duration: $2$ hours, power usage: $130W$) before user comes back from work
  \item Cook dinner (duration: $30$ minutes, power usage: $100W$) ready within 15 minutes of user coming back from work
  \item Have the lights on (power usage: $80W$) from before sunset to at least midnight.
  \item Cook a late night snack (duration: $30$ minutes, power usage: $20W$) between 10pm and 11pm.
  \end{itemize}
}
While probabilistic constraints can be  modeled using probabilistic Simple Temporal Networks \cite{Fang2014} and solved accordingly, there is no known model which captures the tightly coupled resource constraints.

In this paper, we introduce the Time Resource Network (TRN), a general framework capable of encoding scenarios similar to the example described. We describe two algorithms which schedules resource usage given TRN models, one based on a standard encoding as a mixed integer program (MIP) and a novel algorithm leveraging prior specialized algorithms for solving temporal problems. Using the algorithms, we are able to derive a solution to the above example which meets the constraints with $99.7\%$ probability (presented on Figure \ref{fig:pstnu_scheduling}). We also show through benchmarking that the novel algorithm is significantly faster even when the MIP encoding is solved with state-of-the-art commercial solvers.

\begin{figure}[H]
\begin{center}
\includegraphics[width=0.48\textwidth]{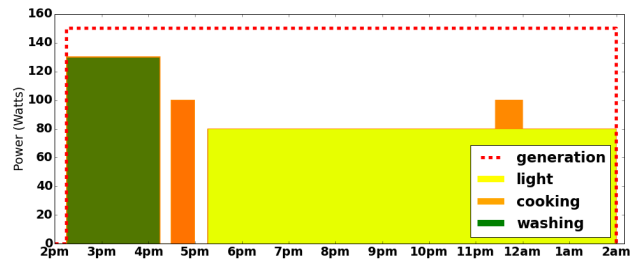}
\caption{Depiction of solution to TRN spanning a pSTN.}
\label{fig:pstnu_scheduling}
\end{center}
\end{figure}

%%%%%%%%%%%%%%%%%%%%%%%%%%%%%% RELATED WORK %%%%%%%%%%%%%%%%%%%%%%%%%%%%%%%%%%%
\section{Related Work}
One of the earliest mentions of a scheduling problem being solved in an algorithmic fashion can be found in \cite{johnson1954optimal}, although there's evidence that the problem was already considered in unpublished versions of \cite{bellman1956mathematical}. This publication considers the following statement of scheduling problem. We have $n$ items and $m$ stages and $A_{i,j}$ denoting the time for $i$-th item to be processed by stage $j$. All the items must be processed by different stages in order (for example first stage is printing of a book and second stage is binding). The publication considers $m=2$ and $m=3$ and arrives at the solution that \textit{``permits one to optimally arrange twenty production items in about five minutes by visual inspection''}. It turns out that the solution to the problem for $m \geq 3$ is NP-hard (\cite{garey1976complexity}). In \cite{wagner1959integer} an Integer Programming solution to the scheduling problem was presented, with a comment stating that it \textit{``is a single model which encompasses a wide variety of machine-scheduling situations''}.

In \cite{pritsker1969multiproject}, a generalization of scheduling problems is considered, which allows for multiple resource constraints. However, the proposed solution uses a discrete time formulation, which, depending on required accuracy, can substantially decrease performance. In 1988 a technique was proposed which can handle resource constraints and continuous time (\cite{bartusch1988scheduling}). The proposed approach can be thought of as resource constrained scheduling over Simple Temporal Networks (STN).

In \cite{dechter1991temporal}, a notion of Simple Temporal Problem was introduced which allows one to solve problems with simple temporal constraints of form $l \leq t_y - t_x \leq u$. This concept was later extended with various more sophisticated notions of temporal constraints. \cite{vidal1996dealing} defined an uncertain temporal constraint, where the duration between two time events can take a value from an interval $[l,u]$, which is unknown during the time of scheduling (uncertain duration constraints). \cite{morris2001dynamic} describes a pseudopolynomial algorithm for handling uncertain duration constraint, where we are allowed to make a scheduling decisions based on knowledge of uncertain durations from the past (Dynamic controllability). The algorithm is later improved to polynomial complexity (\cite{morris2005temporal}). Finally, \cite{Fang2014} provides a non-linear optimization based solver for uncertain temporal constraints where the duration of the constraint can come from arbitrary probabilistic distribution.

%%%%%%%%%%%%%%%%%%%%%%%%%%%%%% PROBLEM STATEMENT %%%%%%%%%%%%%%%%%%%%%%%%%%%%%%
\section{Problem statement}
In this section we introduce a novel formulation - Time Resource Network (TRN). While the results presented in this paper can be extended to multiple types of resources being constrained simultaneously (for example electricity, water, fuel, CPU time and memory among others), for simplicity we consider only one type of constrained resource in this work. Additionally, we only consider the problem of consistency, but the techniques presented can be extended to handle optimization over constrained schedules.

\subsection{Abstract Temporal Network}
We wish to define TRN to support a general class of temporal networks. We thus define the notion of Abstract Temporal Network as a 3-tuple $ATN=\langle E, C, X\rangle$ where $E$ is a set of controllable events, $C$ is a set of simple temporal constraints \cite{dechter1991temporal} and $X$ represents any additional elements such as additional constraints and variables.

\paragraph{Schedule}
A schedule for an $ATN= \langle E,C,X \rangle$ is a mapping $s: \texttt{E} \rightarrow \mathbb{R}$ from events in ATN to their execution times.

\paragraph{Temporal Consistency}
\label{temporal_consistency}
For an $ATN=\langle E,C,X \rangle$ we define a predicate $TC_s(ATN) = stn-consistent(E,C,s) \wedge extra-criteria(E,C,X,s)$, which denotes the $ATN$ is \textbf{temporally consistent} under schedule $s$. $stn-consistent(E,C,s)$ represents STN consistency as defined in \cite{dechter1991temporal}.  $extra-criteria(E,C,X,s)$ depends on the type of the particular ATN. We say that $ATN$ is temporally consistent (denoted by $TC(ATN)$), if there exists at schedule $s$ such that $TC_s(ATN)$.

\paragraph{Example}
An example of a network that satisfies the ATN interface is Simple Temporal Network with Uncertainty (STNU) described in \cite{vidal1996dealing}. The set $E$ is composed of all the activated and received events, $C$ is the set of requirement links, $X$ is the set of all the contingent links. One way to define is $TC(ATN)$ is to be true if and only if the networks is strongly controllable (which already implies $stn-consistent(E,C,s)$).

%Let's consider cc-pSTP \cite{Fang2014} as an example. Here \texttt{events} returns set of \textit{activated} and \textit{received} time points. \texttt{extend} returns network with extra \textit{free contraints} encoding the simple temporal constraints. The temporal consistency check $TC$ is true if cc-pSTP has a solution.
\subsection{Time Resource Network}
\label{sec:trn_definition}
A Time Resource Network is described by a tuple $TRN = \langle ATN, R \rangle$, where $ATN$ is an Abstract Temporal Network and $R={src_1, ..., src_n}$ is a set of \textbf{simple resource constraints}, each of which is a triplet $ \langle x, y, r \rangle$, where $x, y \in$ E and $r \in \mathbb{R}$ is the amount of resource, which can be positive (consumption) and negative (generation). Given a schedule $s$ for any time $t \in \mathbb{R}$ we define \textbf{resource usage} for $src=\langle x,y,r \rangle$ as:
\begin{align*}
u_s(src, t) = \begin{cases}
r & \text{if}\ s(x) \leq t < s(y)\\
0 & \text{otherwise}
\end{cases}
\end{align*}
Intuitively, simple resource constraint encodes the fact that between time $s(x)$ and $s(y)$  resource is consumed (generated) at the rate $|r|$ per unit time for positive (negative) $r$.

Our notation is inspired by \cite{bartusch1988scheduling}. The authors have demonstrated that it is possible encode arbitrary piecewise-constant resource profile, by representing each constant interval by a simple resource constraint and joining ends of those intervals by simple temporal constraints.

\subsection{Resource consistency}
For a schedule $s$ we define a \textbf{net-usage} of a resource at time $t \in \mathbb{R}$ as:
\[
U_s(t) = \sum_{\forall_{src_i \in R}} u_s(src_i, t)
\]
$R$ is the set of all the resource constraints. We say that the network is \textbf{resource consistent} under schedule $s$ when it satisfies predicate $RC_s(TRN)$, i.e.
\begin{align}
\label{usage_for_all}\forall_{t \in \mathbb{R}} . U_s(t) \leq 0
\end{align}
Intuitively, it means that resource is never consumed at a rate that is greater than the generation rate. We say that $TRN$ is resource consistent, if there exists $s$, such that $RC_s(TRN)$ is true.

\subsection{Time-resource consistency}
$TRN=(ATN, R)$ is \textbf{time-resource consistent} if there exists a schedule $s$ such that $RC_s(TRN) \wedge TC_s(ATN)$. Determining whether a $TRN$ is time-resource consistent is the central problem addressed in this publication.

\subsection{Properties of TRN}
Before we proceed to describe algorithms for determining time-resource consistency it will be helpful to understand some properties common to every TRN.
\begin{lemma}
\label{resource_checking}
For a $TRN$ a schedule $s$ is resource consistent if and only if
\begin{align}
\label{eq:resource_consistency} \forall_{e \in E} U_s(s(e)) \leq 0
\end{align}
i.e. resource usage is non-positive a moment after all of the scheduled events.
\end{lemma}
\begin{proof}

$\Rightarrow$ Follows from definition of resource-consistency.\\
$\Leftarrow$ We say a time point $t \in \mathbb{R}$ is scheduled if there exists an event  $e \in E$ such that $t = s(e)$. Assume for contradiction, that the right side of the implication is satisfied, but the schedule is not resource consistent. That means that there exists a time point $t_{danger}$ for which $U_s(t_{danger}) > 0 $. Notice that by assumption $t_{danger}$ could not be scheduled. Let $t_{before}$ be the highest scheduled time point is smaller than $t_{danger}$. Notice that if no such time point existed, that would mean that there is no resource constraint $(x,y,r)$ such that $s(x) \leq t_{danger} < s(y)$, so $U_s(t_{danger})=0$ . By assumption, $U_s(t_{before}) < 0$.  We can therefore assume that $t_{before}$ exists. Notice that by definition of $t_{before}$ and simple resource constraints, $U_s(t)$ for $t_{before} \leq t \leq t_{danger}$ is constant. If it wasn't there would be another scheduled point between $t_{before}$ and $t_{danger}$, but we assumed that $t_{before}$ is highest scheduled point smaller than $t_{danger}$. Therefore $ U_s(t_{danger}) = U_s(t_{before})$. But we assumed that $U_s(t_{danger}) > 0$ and $U_s(t_{before}) < 0 $ Contradiction.
\end{proof}
\begin{corollary}
\label{cor:ordering}
Given a $TRN$ and two schedules $A$ and $B$ where all events occur in the same order, $A$ is resource consistent if and only if $B$ is resource consistent.
\end{corollary}
\begin{proof}
Notice that if we move execution time of arbitrary event, while preserving the relative ordering of all the events, then net resource usage at that event will not change. Therefore by lemma \ref{resource_checking},  $A$ is resource-consistent if and only if $B$ is resource-consistent.
\end{proof}

%\subsection{Common scheduling problems expressed as Time Resource Network}

%%%%%%%%%%%%%%%%%%%%%%%%%%%%%% ALGORITHM %%%%%%%%%%%%%%%%%%%%%%%%%%%%%%%%%%%%%%
\section{Approach}
In this section we present two approaches for determining time-resource consistency of a TRN. One of them involves Mixed Integer Programming (MIP) and the other Constraint Problem (CP) formulations.
\subsection{Definitions}
Let's take a $TRN=\langle ATN, R \rangle$ where $R={src_1, ..., src_n}$ and $src_i = \langle x_i, y_i, r_i \rangle$ as defined in section \ref{sec:trn_definition}. Let's denote all the events relevant for resource constraints as $RE \subseteq E$, i.e.
\begin{align*}
RE = \{ x_i | \langle x_i, y_i, r_i \rangle \in R \} \cup \{ y_i | \langle x_i, y_i, r_i \rangle \in R \}
\end{align*}
Additionally, let's introduce resource-change at event $e \in E$ as:
\begin{align*}
\Delta(e) = \sum_{\langle x_i, y_i, r_i \rangle \in R, x_i = e} r_i + \sum_{ \langle x_i, y_i, r_i \rangle \in R, y_i = e} -r_i
\end{align*}
Intuitively $\Delta(n)$ is the amount by which resource usage changes after time $s(n)$ under schedule $s$.

\subsection{Mixed Integer Programming based algorithm}
Mixed Integer Programming (\cite{markowitz1957solution}) allows one to express scheduling problems in an intuitive way. In this section we present a way to formulate TRN as a MIP problem. The technique is very similar to the ones used in state of the art solvers for general scheduling  \cite{patterson1984comparison} \cite{bartusch1988scheduling}. Therefore, the purpose of this section is not to introduce a novel approach, but to demonstrate that those algorithms are straightforward to express using TRN formulation. Let \texttt{TC-formulation(ATN)} be a MIP-formulation that has a solution if an only if $TC(ATN)$. For some types of $ATN$ such a formulation might not exist and in those cases MIP-based algorithm cannot be applied.

The following MIP program has a solution if and only if the TRN is time-resource-consistent:
\begin{align}
\label{eq:mip0} & \forall_{e \in E}.              & 0 \leq e \leq M \\
\label{eq:mip1} & \forall_{e_1, e_2 \in RE, e_1 \neq e_2}. & e_1 - e_2 \geq - x_{e_1,e_2} M \\
\label{eq:mip2} & \forall_{e_1, e_2 \in RE, e_1 \neq e_2}. & e_1 - e_2 \leq (1.0 - x_{e_1,e_2}) M\\
\label{eq:mip3} & \forall_{e_1, e_2 \in RE, e_1 \neq e_2}. & x_{e_1,e_2} + x_{e_2,e_1}  = 1\\
\label{eq:mip4} & \forall_{e_1, e_2 \in RE, e_1 \neq e_2}. & x_{e_1,e_2} \in \{ 0, 1 \} \\
\label{eq:mip5} & \forall_{e_1 \in RE}.                    & \sum_{e_2 \in RE} x_{e_2, e_1} \Delta(e_2) \leq 0\\
\label{eq:mip6} & \texttt{TC-formulation(ATN)}
\end{align}

Variable $M$ denotes the time horizon, such that all the variables are scheduled between $0$ and $M$. This definition is imposed in eq. \ref{eq:mip0}.
Variables $x_{e_1,e_2}$ are order variables, i.e.
\begin{align*}
x_{e_1, e_2} = \begin{cases}
1 &\text{ if }s(e_1) \leq s(e_2) \\
0 &\text{ otherwise}
\end{cases}
\end{align*}
Equations \ref{eq:mip1}, \ref{eq:mip2}, \ref{eq:mip3}, \ref{eq:mip4} enforce that definition. In particular equations \ref{eq:mip1}, \ref{eq:mip2} enforce the ordering using big-$M$ formulation that is correct because of time horizon constraint. In theory eq. \ref{eq:mip3} could be eliminated by careful use of $\epsilon$ (making sure no two timepoints are scheduled at exactly the same time), but we found that in practice they result in useful cutting planes that decrease the total optimization time. Equation \ref{eq:mip5} ensures resource consistency by lemma \ref{resource_checking}. Finally eq. \ref{eq:mip6} ensures time consistency.

Solving that Mixed-Integer Program will yield a valid schedule if one exists, which can be recovered by inspecting values of variables $t \in E$.

\subsection{Constraint Programming based algorithm}
The downside of MIP approach is the fact that the ATN must have a MIP formulation (e.g. pSTN does not have one). In this section we present a novel CP approach which addresses this concern. The high level idea of the algorithm is quite simple and is presented in algorithm \ref{hl_algo}. In the second line, we iterate over all the permutations of the events. On line 3 we use \texttt{resource\_consistent} function to check resource consistency, which by corollary \ref{cor:ordering} is only dependent on the chosen permutation. On line five we use $TC$ checker to determine if network is time consistent - the implementation depends on the type of $ATN$ and we assume it is available. Function $encode\_as\_stcs$ encodes permutation using simple temporal constraints. For example if $\sigma(1) = 2$ and $\sigma(2) = 1$ and $\sigma(3) = 3$, then we can encode it by two STCs: $ 2 \leftarrow 1 $ and $1 \leftarrow 3$.

\begin{algorithm}[h]
    \label{hl_algo}
    \KwData{$TRN=\langle ATN, R \rangle$, $ATN= \langle E,C,X \rangle $}
    \KwResult{true if TRN is time-resource-consistent}
    $N \leftarrow E$\;
    \For{$\sigma \leftarrow \text{permutation of } N$}{
        \If{\texttt{resource\_consistent(R, $\sigma$)} }{
            $ATN' = (E,C \cup \text{encode\_as\_stcs}(\sigma), X)$ \;
            \If{TC(ATN)}{
                \Return true\;
            }
        }
    }
    \Return false\;
    \caption{Time-resource-consistency of a TRN }
\end{algorithm}

\vspace{10mm}

The implementation of \texttt{resource\_consistent} follows from lemma \ref{resource_checking} and is straightforward - we can evaluate $U_s(s(e))$ for all events $e \in RE$ (which can be done only knowing their relative ordering), and if it is always non-positive then we return true.

To improve the performance w.r.t algorithm \ref{hl_algo}, we use off-the-shelf constraint propagation software (PyConstraint). Let's consider $RE={e_1, ..., e_N}$. We define a problem using $N$ variables:  $x_1, x_2, ..., x_N \in \{ 1, ..., N \}$, such that $x_j=i$ if $e_i$ is $j$-th in the temporal order, i.e. $x_1, ..., x_N$ represent the permutation $\sigma$. We used the following pruners which, when combined, make the CP solver behave similarly to algorithm \ref{hl_algo}, but ignoring some pruned permutations:
\begin{itemize}
\setlength\itemsep{0.2em}
\item \textbf{all\_different\_constraint} - ensure that all variables are different, i.e. they actually represent a permutation. This is standard constraint available in most CP software packages.
\item \textbf{time\_consistent} - making sure that the temporal constraints implied by the permutation are not making the $ATN$ inconsistent. Even when the variables are partially instantiated, we can compute a set of temporal constraints implied by the partially instantiated permutation. For example if we only know that $x_1 = 3$, $x_5 = 2$ and $x_6=5$, it implies $e_5 \leq e_1 \leq e_6$.
\item \textbf{resource\_consistent} - ensure that for all $e_1, ..., e_n \in RE$, resource usage just after $e_i$ is non-positive. Even if the order is partially specified we can still evaluate it. A subtlety which needs to be considered is that we need to assume that all the events for which $x_i$ is undefined and which are generating ($\delta(e_i) < 0$) could be scheduled before all the points for which order is defined. For example if $n = 4$ and $\Delta(e_1) = 4$, $\Delta(e_2) = -6$, $\Delta(e_3) = 3$, $\Delta(e_4) = 4$ and we only know that $x_1 = 3$, $x_3 = 2$, then we have to assume that all the generation happened before the points that we know, i.e. initially resource usage is $-6$, then after $e_3$ is is $-3$, and after $e_1$ it is $1$, therefore violating the constraint. But if in that scenario we would instead have $\Delta(e_1) = 2$ and we hadn't had assumed that all the unscheduled generation $-6$ happens at the beginning, we would have falsely deduced that the given variable assignment could never be made resource consistent.
\end{itemize}

\subsubsection{TRN limitations - Going Beyond Fixed Schedules}
Notice that CP algorithm does not require the schedule to be fixed. For example, we could consider $ATN$ to be $STNU$ and $TC$ to be dynamic controllability (\cite{vidal1996dealing}). There, we seek an execution strategy, rather than a schedule. While this can be implemented for a TRN, there is an important limitation to that approach. Even though temporal schedule is dynamic, the schedule implied by resource constraints is static - we cannot change $\sigma$ dynamically during execution.

\begin{figure}[H]
\begin{center}
\includegraphics[width=0.48\textwidth,trim={0.23cm 0.23cm 0.00cm 0.37cm},clip]{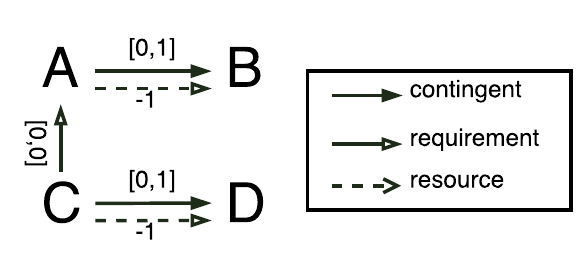}
\caption{TRN cannot select $\sigma$ dynamically. Number below a  simple resource constraint represents $r$.}
\label{fig:stnu_counter}
\end{center}
\end{figure}

\vspace{10mm}

Figure \ref{fig:stnu_counter} shows an example where $TRN$ would report no solution found. However, if we ignore the resource constraints and find a dynamic execution strategy satisfying temporal constraints, it never violates the resource constraints, as they are both generating. The reason TRN fails to find the solution is due to the fact that $B$ and $D$ are both in the set $RE$ and TRN's solution attempts to fix the ordering between $B$ and $D$, which is impossible to do statically in this example.

% INPUT: atn N, {src} S (spanning N.timepoints)
% OUTPUT: scheduling strategy on N or fail
% ALGORITHM:
% X = subset of N.timepoints used by SRCs from S
% for every permutation pi of X:
% stcs = pi encoded by STCs
% result = N.solve_with_stcs(stcs)
% if result is schedule:
%        return schedule
%       fail

%%%%%%%%%%%%%%%%%%%%%%%%%%%%%% EXPERIMENTS %%%%%%%%%%%%%%%%%%%%%%%%%%%%%%%%%%%%
\section{Experiments}

\begin{figure*}
\begin{center}
\includegraphics[width=\textwidth]{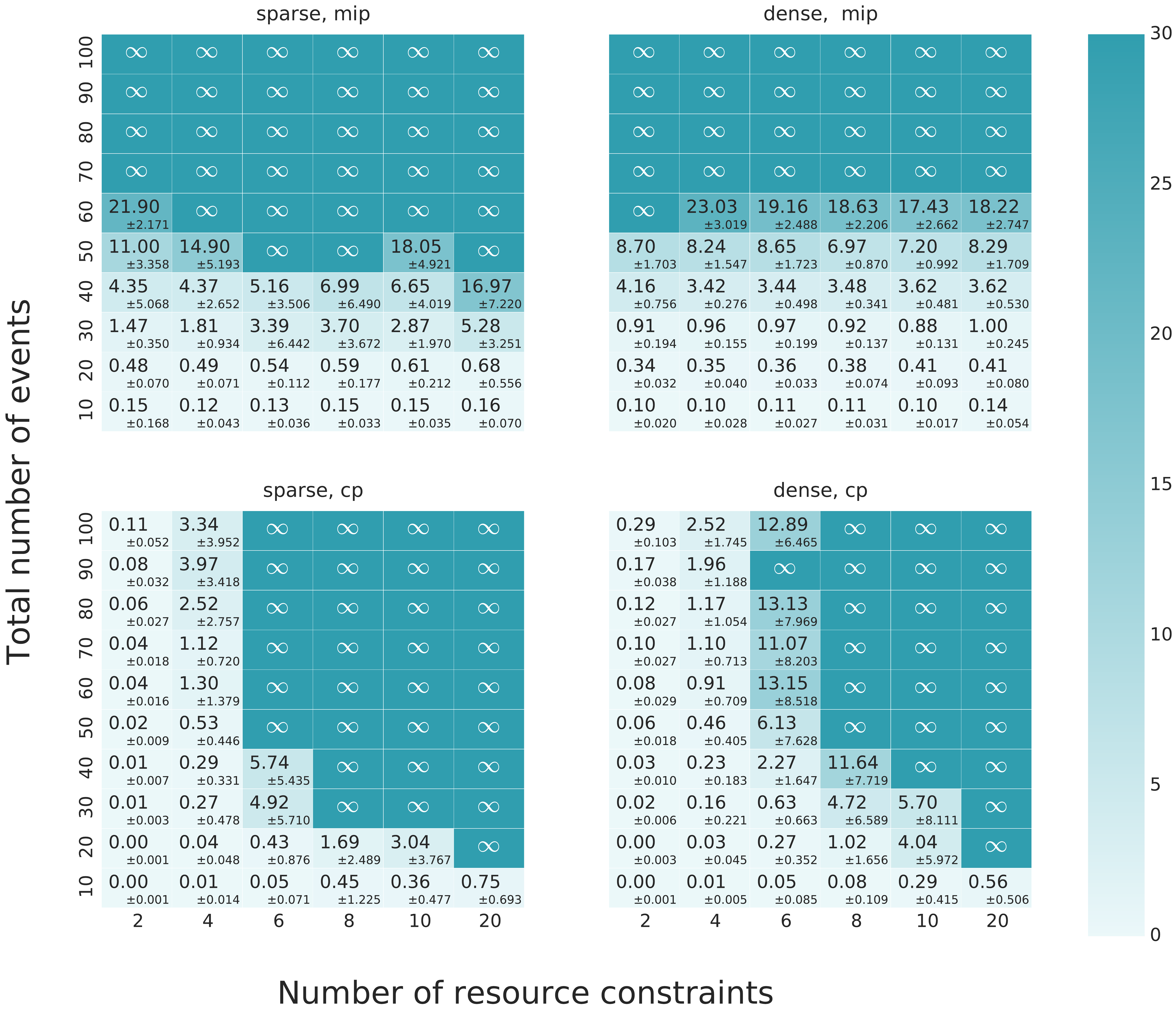}
\caption{Comporison of execution time for different types of networks, or $\infty$ if the solver failed to compute the result within the time limit. Y axis represents the number of events in the temporal network ($N$). X axis represents the number of resource constraints ($R$). Top portion of the figure was obtained using the MIP-based solver, while bottom part of the figure was obtained using CP-based solver. The left side of the figure represents computations on \textit{sparse} networks, which in this case means that the total number of temporal constraints is $2N$. On the right side we have \textit{dense} networks, meaning that the number of temporal constraints is $N^2/2$. This figure was computed by running the experiment for every set of parameters multiple times, but each time with different randomly generated instance. Numbers in bottom right corner of each cell are corresponding standard deviations.}
\label{fig:execution_time}
\end{center}
\end{figure*}

\begin{figure*}
\begin{center}
\includegraphics[width=\textwidth]{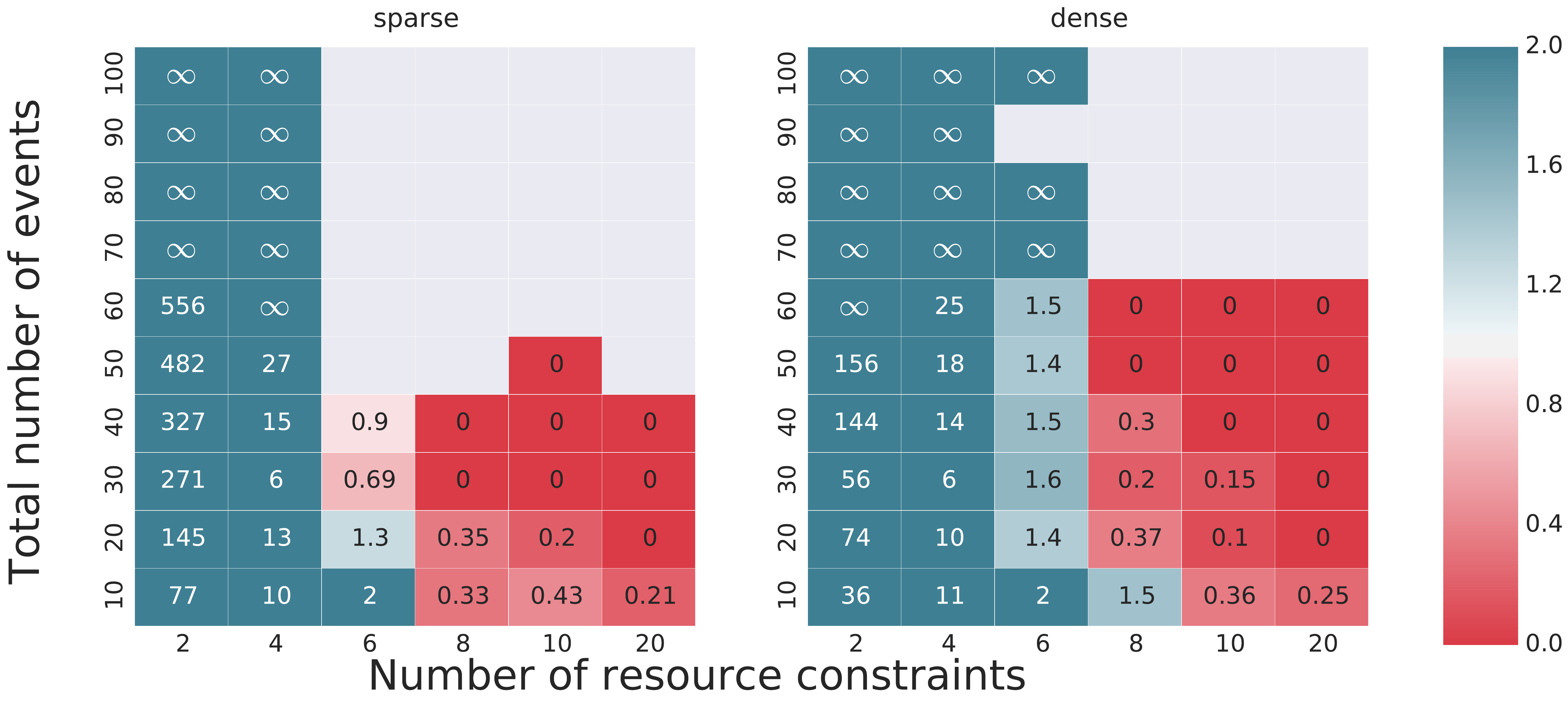}
\caption{Number on the figure represents execution time using MIP-based algorithm divided by execution time using CP-based algorithm. Notice that in particular $0$, means that CP-based algorithm failed to compute the results within the time limit and $\infty$ means that MIP-based algorithm timed out. The missing cells correspond to the networks where both of the algorithms timed out and therefore their execution time cannot be compared.   }
\label{fig:when_better}
\end{center}
\end{figure*}

\subsection{TRN over STN}
To understand the performance of our novel CP algorithm, we used the proposed MIP approach as a baseline. We used Gurobi as a MIP solver. Both algorithms were used to determine time-resource consistency for TRN over Simple Temporal Network. In case of MIP based algorithm, all the temporal constraints $l \leq x - y \leq u$, where $l,b \in \mathbb{R}$ and $x,y \in E$ can be expressed as linear constraints, with $x$ and $y$ being continuous variables. In case of CP algorithm, we used Floyd-Warshall to determine temporal consistency as suggested in \cite{dechter1991temporal}. The test cases were created by the following procedure:

% \small{
  \begin{enumerate}
  \setlength\itemsep{0.1em}
  \item Specify number of events $N \geq 2$, number of temporal constraints $T\geq 2$ and number of resource constraints $R\geq 2$
  \item Create a random schedule $s$ for events in $N$ with times in the interval $(0.0, 1.0)$.
  \item Create $T$ time constraints using the following procedure:
    \begin{enumerate}
    \item Choose start and end points $x,y \in N$.
    \item Choose a type of constraint - lower bound or upper bound, each with probability $0.5$
    \item Let $d=s(y) - s(x)$ and chose number $d'$ form exponential distribution with $\lambda = 1 / \sqrt{d}$. For lower-bound set $l = d - d'$. For upper bound set $u = d + d'$.
    \end{enumerate}
  \item Choose number of generating constraints $G$ as a random integer between $1$ and $R-1$ and set number of consuming constraints as $C = R - G$ (so that there's at least on constraint of each type).
  \item Create $G$ generating constraints using the following procedure, by randomly choosing $x,y \in N$ and setting $r$ to a random number between $-1$ and $0$.
  \item Create $C$ consuming constraints using the following procedure.
    \begin{enumerate}
    \item Choose start and end points $x,y \in N$.
    \item Let $m$ be the maximum resource usage value between $x$ and $y$ considering all the resource constraints generated so far. If $m = 0$ repeat the process.
    \item choose $r$ from uniform distribution between $0$ and $-m$.
    \end{enumerate}
  \end{enumerate}
% }

We considered $10$ different values of $N$: $10, 20, ..., 100$. We considered $6$ different values of $R$: $2, 4, 6, 8, 10, 20$. We defined two types of networks - sparse, where $T = 2N$ and dense where $T = N^2/2$. For every set of parameters we run $5$ trials. We set the time limit to $30$ seconds. The results are presented on figure \ref{fig:execution_time}. We can see there exists a set of parameters where only CP managed to find the solution MIP exceed the time limit and vice versa. Figure \ref{fig:when_better} compares execution time of CP and MIP algorithms. The cells colored in blue are the ones where CP algorithm is faster and the cells colored in red are the ones where MIP based algorithm is better. One can see that CP is much better suited for large temporal networks with small number of resource constraints, while MIP scales much better with the number of resource constraints.

\subsection{TRN over pSTN}
To demonstrate extensibility of our approach we have implemented a version of TRN network, where the underlying temporal network is a pSTN (\cite{Fang2014}). pSTN extends the notion of STN. It defines STN-like events and edges as \textbf{actiavated time points} and \textbf{free constraints} respectively. It extends STN with \textbf{received time points}, which are determined by the environment. Every received time point is defined by corresponding \textbf{uncertain duration (uDn)} constraint, which specifies a probability distribution over duration between some activated time point and the received time point. Due to that extension, the notion of consistency $TC(ATN)$ becomes probabilistic; rather than asking \textit{is this pSTN consistent?}, we ask is \textit{is this pSTN consistent with probability $p$?}. Since pSTN is an extension of STN, it is an $ATN$. Given the choice of $p$ we can use probabilistic consistency as $TC$. Therefore we can use CP algorithm to check networks consistency. Example scenario and the schedule obtained by the algorithm is presented in the introduction.

%%%%%%%%%%%%%%%%%%%%%%%%%%%%%% FUTURE WORK %%%%%%%%%%%%%%%%%%%%%%%%%%%%%%%%%%%%
% \section{Future Work}
% \textbf{linear resource constraint} is a triplet $(x, y, r_b, r_e)$, where $x, y \in \texttt{nodes(ATN)}$ and resource usage at time $s(x) \leq t \leq s(y)$ is equal to
% \[
%     u(t) = r_b + t  \frac{r_e - r_b}{s(y) - s(x)}
% \]
% Intuitively, simple resource constraint encodes the fact that between time $s(x)$ and $s(y)$  resource is consumed/generated with rate that changes linearly between $s(x)$ and $s(y)$.
%
% \textbf{probabilistic simple resource constraint}
% Is an extension of simple resource constraint where $r$ is a random variable (and therefore so is $u(t)$).

%%%%%%%%%%%%%%%%%%%%%%%%%%%%%% CONCLUSION %%%%%%%%%%%%%%%%%%%%%%%%%%%%%%%%%%%%%
\section{Conclusion}
In this paper, we have introduced Time Resource Networks, which allow one to encode many resource-constrained scheduling problems. We defined them in a way that permits use of many different notions of temporal networks to constrain schedules. We introduced a novel CP algorithm for determining time-resource consistency of a TRN and we compared it MIP baseline. We have demonstrated that our algorithm achieves superior performance for networks with large number of temporal constraints and small number of resource constraints. In addition, we have shown that CP algorithm is flexible and can support recently introduced probabilistic Simple Temporal Networks \cite{Fang2014}.

%%%%%%%%%%%%%%%%%%%%%%%%%%%%%% REFERENCES %%%%%%%%%%%%%%%%%%%%%%%%%%%%%%%%%%%%%

\bibliographystyle{named}
\bibliography{scheduling}

\end{document}